\documentclass{article}
\pdfoutput=1
\usepackage{amsthm} 

\usepackage{ijcai22}

\usepackage{times}

\usepackage{soul}
\usepackage{url}
\usepackage[hidelinks]{hyperref}
\usepackage[utf8]{inputenc}
\usepackage[small]{caption}
\usepackage{graphicx}
\usepackage{amsmath}
\usepackage{booktabs}
\usepackage{thm-restate}

\usepackage[colorinlistoftodos,textsize=tiny]{todonotes} 
\newcommand{\Comments}{1}
\newcommand{\mynote}[2]{\ifnum\Comments=1\textcolor{#1}{#2}\fi}
\newcommand{\mytodo}[2]{\ifnum\Comments=1%
  \todo[linecolor=#1!80!black,backgroundcolor=#1,bordercolor=#1!80!black]{#2}\fi}

\ifnum\Comments=1               
\fi

\usepackage{natbib}
\usepackage{algorithm}
\usepackage[noend]{algorithmic}
\usepackage{amsmath}
\usepackage{amssymb}
\usepackage{bm}
\usepackage{dsfont} 
\usepackage{subcaption}

\usepackage{mathtools}
\usepackage{enumitem}
\usepackage{tikz}
\usepackage{pgfplots}
\usepackage{pgfplotstable}
\usepgfplotslibrary{fillbetween}
\pgfplotsset{compat=1.16}
\pgfdeclarelayer{ft}
\pgfdeclarelayer{bg}
\pgfsetlayers{bg,main,ft}

\theoremstyle{definition}

\newcommand{\density}[1]{d_{#1}}
\newcommand{\budget}{B}
\newcommand{\truerank}{\texttt{rank}}

\newcommand{\observedreward}{\texttt{reward}}
\newcommand{\truerewardvec}{\mu}
\newcommand{\truereward}[1]{\mu_{#1}}

\newcommand{\approxreward}[1]{\hat{\mu}_{#1}}
\newcommand{\approxrewardij}[2]{\hat{\mu}_t({#1}, {#2})}

\renewcommand{\epsilon}{\varepsilon}

\newcommand{\dist}{\mathcal{P}}

\newcommand{\tunepriority}{\lambda}

\newcommand{\effortvec}{\vec{\beta}} 
\newcommand{\opteffortvec}{\vec{\beta}^\star} 
\newcommand{\opteffort}[1]{\beta_{#1}^\star} 
\newcommand{\effort}[1]{\beta_{#1}} 
\newcommand{\discretization}{\Psi} 
\newcommand{\discretizationlevel}[1]{\psi_{#1}} 

\newcommand{\selfucb}{\textsc{selfUCB}} 
\newcommand{\ucb}{\textsc{UCB}} 

\newcommand{\obj}{\texttt{obj}(\effortvec)} 

\newcommand{\priority}[1]{\texttt{benefit}(\effortvec;#1)} 

\usepackage{xcolor}

\newif\ifshowplot
 \showplotfalse

\title{Ranked Prioritization of Groups in Combinatorial Bandit Allocation}

\author{
Lily Xu$^1$\and
Arpita Biswas$^1$\and
Fei Fang$^{2}$\And
Milind Tambe$^1$\\
\affiliations
$^1$Center for Research on Computation and Society, Harvard University\\
$^2$Institute for Software Research, Carnegie Mellon University\\
\emails
lily\_xu@g.harvard.edu, arpitabiswas@seas.harvard.edu, feif@cs.cmu.edu, milind\_tambe@harvard.edu
}

\begin{document}

\maketitle

\begin{abstract}
Preventing poaching through ranger patrols protects endangered wildlife, directly contributing to the UN Sustainable Development Goal 15 of life on land. 
Combinatorial bandits have been used to allocate limited patrol resources, 
but existing approaches overlook the fact that each location is home to multiple species in varying proportions, so a patrol benefits each species to differing degrees. 
When some species are more vulnerable, we ought to offer more protection to these animals; unfortunately, existing combinatorial bandit approaches do not offer a way to prioritize important species. 
To bridge this gap, (1)~We propose a novel combinatorial bandit objective that trades off between reward maximization and also accounts for prioritization over species, which we call \emph{ranked prioritization}. We show this objective can be expressed as a weighted linear sum of Lipschitz-continuous reward functions. 
(2)~We provide RankedCUCB,\footnote{Code is available at \url{https://github.com/lily-x/rankedCUCB}} an algorithm to select combinatorial actions that optimize our prioritization-based objective, and prove that it achieves asymptotic no-regret. 
(3)~We demonstrate empirically that RankedCUCB leads to up to 38\% improvement in outcomes for endangered species using real-world wildlife conservation data. Along with adapting to other challenges such as preventing illegal logging and overfishing, our no-regret algorithm addresses the general combinatorial bandit problem with a weighted linear objective. 
\end{abstract}

\section{Introduction}

More than a quarter of mammals assessed by the IUCN Red List are threatened with extinction \citep{gilbert2008quarter}. As part of the UN Sustainable Development Goals, Target~15.7 focuses on ending poaching, and Target~15.5 directs us to halt the loss of biodiversity and prevent the extinction of threatened species.
To efficiently allocate limited resources, multi-armed bandits (MABs), and in particular combinatorial bandits \citep{chen2016combinatorial,cesa2012combinatorial}, have been widely used for a variety of tasks \citep{bastani2020online,segal2018combining,misra2019dynamic} including ranger patrols to prevent poaching~\citep{xu2021dual}. In this poaching prevention setting, the patrol planner is tasked with repeatedly and efficiently allocating a limited number of patrol resources across different locations within the park \citep{plumptre2014efficiently,fang2016deploying,xu2021robust}.




\begin{figure}
\centering
 \includegraphics[bb=0 0 599 141, width=\linewidth]{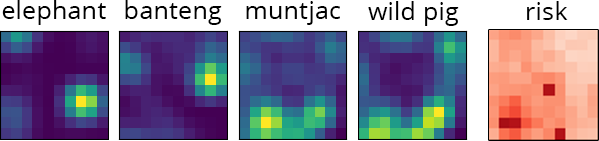}
\caption{Animal density distributions and poaching risk (expected number of snares) 
across a $10 \times 10$~km region in Srepok Wildlife Sanctuary, Cambodia. 
The distribution of the species of greatest conservation concern, the elephant, differs from that of other animals and overall poaching risk.}
\label{fig:animal-density}
\end{figure}

In past work, we worked with the World Wide Fund for Nature (WWF) to deploy machine learning approaches for patrol planning in Srepok Wildlife Sanctuary in Cambodia \citep{xu2020stay}. Subsequent conversations with park managers and conservation biologists raised the importance of focusing on conservation of \emph{vulnerable} species during patrols, revealing a key oversight in our past work. In this paper, we address these shortcomings to better align our algorithmic objectives with on-the-ground conservation priorities.

Looking at real-world animal distribution and snare locations in Srepok visualized in Figure~\ref{fig:animal-density}, we observe that the locations that maximize expected reward, defined as finding the most snares (darkest red in the risk map), are not necessarily the regions with high density of endangered animals (elephants). 
To effectively improve conservation outcomes, it is essential to account for these disparate impacts, as the relatively abundant muntjac and wildpig would benefit most if we simply patrol the regions with most snares, neglecting the endangered elephant. 
Prioritization of species is well-grounded in existing conservation best practices \citep{regan2008species,arponen2012prioritizing,dilkina2017trade}. The IUCN Red List of Threatened Species, which classifies species into nine groups from critically endangered to least concerned, is regarded as an important tool to focus conservation efforts on species with the greatest risk of extinction \citep{rodrigues2006value}. We term the goal of preferentially allocating resources to maximize benefit to priority groups as \textit{ranked prioritization}.
Some existing multi-armed bandit approaches have considered priority-based fairness \citep{joseph2016fairness,kearns2017meritocratic,schumann2022group}. However, unlike ours, these prior works only consider stochastic, not combinatorial, bandits. In our combinatorial bandit problem setup, we must determine some hours of ranger patrol to allocate across $N$ locations at each timestep, subject to a budget constraint, 
The rewards obtained from these actions are unknown \textit{a priori} so must be learned in an online manner. Existing combinatorial bandit approaches \citep{chen2016combinatorial,xu2021dual} achieve no-regret guarantees with respect to the objective of maximizing the direct sum of rewards obtained from all the arms. However, we wish to directly consider ranked prioritization of groups in our objective; learning combinatorial rewards in an online fashion while ensuring we justly prioritize vulnerable groups as well as reward makes the problem more challenging. \textit{How to trade off reward and prioritization in the combinatorial bandit setting} has so far remained an open question. We show experimentally that straightforward solutions fail to make the appropriate trade-off.
To improve wildlife conservation outcomes by addressing the need for species prioritization, we contribute to the literature on online learning for resource allocation across groups with the following:
\begin{enumerate}[leftmargin=*,topsep=3pt]
    \item We introduce a novel problem formulation that considers resource allocation across groups with ranked prioritization, which has significant implications for wildlife conservation, as well as illegal logging and overfishing. We show that an objective that considers both prioritization and reward can be expressed as a weighted linear sum, providing the useful result that the optimal action in hindsight can be found efficiently via a linear program. 
    \item We provide a no-regret learning algorithm, RankedCUCB, for the novel problem of selecting a combinatorial action (amount of patrol effort to allocate to each location) at each timestep to simultaneously attain high reward and group prioritization. We prove that RankedCUCB achieves sub-linear regret of $O(\frac{\ln{T}}{NT} +  \frac{NJ}{T} )$ for a setting with $N$~locations, $O(J^N)$ combinatorial actions, and time horizon~$T$. 
    \item Using real species density and poaching data from Srepok Wildlife Sanctuary in Cambodia along with synthetic data, we experimentally demonstrate that RankedCUCB achieves up to 38\% improvement over existing approaches in ranked prioritization for endangered wildlife. 

\end{enumerate}

\section{Problem formulation}
\label{sec:model}

Consider a protected area with $N$ locations and $G$ groups of interest, with each group representing a species or a set of species in the context of poaching prevention. Let $\density{gi}$ denote the (known) fraction of animals of a group $g\in [G]$ present in location $i \in [N]$. 
Note that $\sum_{i=1}^{N} \density{gi} = 1$ for all $g \in [G]$. We assume the groups are disjoint, i.e., each animal is a member of exactly one group, and the impact of an action on a location equally benefits all animals that are present there. 

At each timestep $t$ until the horizon~$T$, the planner determines an action, which is an \emph{effort vector} $\effortvec^{(t)} = (\effort{1}, \ldots, \effort{N})$ specifying the amount of effort (e.g., number of patrol hours) to allocate to each location. 
The total effort at each timestep is constrained by a budget~$\budget$ such that $\sum_{i=1}^{N} \effort{i} \leq \budget$. To enable implementation in practice, we assume the effort is discretized into $J$~levels, thus $\effort{i} \in \discretization=\{\discretizationlevel{1}, \ldots, \discretizationlevel{J}\}$ for all~$i$.

After taking action~$\effortvec$, the decision maker receives some reward $\truerewardvec(\effortvec)$. We assume the reward is decomposable \citep{chen2016combinatorial}, defined as the sum of the reward at each location:
\begin{align}
    \textit{expected reward} = \truerewardvec(\effortvec) = \sum_{i=1}^{N} \truereward{i}(\effort{i}) \ \label{eq:expected_reward}.
\end{align}
We assume $\truereward{i} \in [0, 1]$ for all~$i$.
In the poaching context, the reward $\truereward{i}$ corresponds to the true probability of detecting snares at a location~$i$. 
The reward function is initially unknown, leading to a learning problem. 
Following the model of combinatorial bandit allocation from \citet{xu2021dual}, we assume that the reward function $\truereward{i}(\cdot)$ is (1)~Lipschitz continuous, i.e., $|\truereward{i}(\discretizationlevel{j}) - \truereward{i}(\discretizationlevel{k})| \leq L \cdot |\discretizationlevel{j} - \discretizationlevel{k}|$ for some Lipschitz constant $L$ and all pairs $\discretizationlevel{j}, \discretizationlevel{k}$, and (2)~monotonically nondecreasing such that exerting greater effort never results in a decrease in reward, i.e., $\effort{i}^\prime > \effort{i}$ implies $\truereward{i}(\effort{i}^\prime) \geq \truereward{i}(\effort{i})$. 




Finally, we assume that an ordinal ranking is known to indicate which group is more vulnerable. The ranking informs the planner which groups are of relatively greater concern, as we often lack sufficient data to quantify the extent to which one group should be prioritized. Without loss of generality, assume the groups are numerically ordered with $g=1$ being the most vulnerable group and $g=G$ being the least, so the true rank is $\truerank = \langle 1, \ldots, G \rangle$.

\subsection{Measuring ranked prioritization}
\label{sec:compare-rank}
To evaluate prioritization, we must measure how closely the outcome of a combinatorial  action~$\effortvec$ aligns with the priority ranking over groups. %
Later in Section~\ref{sec:calibrated-fairness}, we discuss how our approach can generalize to other prioritization definitions, such as the case where have greater specificity over the relative prioritization of each group.



To formalize prioritization, we define $\priority{g}$ which quantifies the benefit group~$g$ receives from action~$\effortvec$. As mentioned earlier, we assume that reward $\truereward{i}(\effort{i})$ obtained by taking an action $\effort{i}$ impacts all individuals at location~$i$. Let $\eta_{gi}$ be the number of individuals from group~$g$ at location~$i$ and $\eta_g$ be the total number of individuals in group~$g$. We define
\begin{align}
    \priority{g} :=& \frac{\sum\limits_{i=1}^{N} \eta_{gi}\ \truereward{i}(\effort{i})}{\eta_g}\nonumber \\
    \triangleq& \sum_{i=1}^{N}  d_{gi}\ \truereward{i}(\effort{i}) \ .
    \label{eq:benefit}
\end{align}
An action $\effortvec$ perfectly follows ranked prioritization when
\begin{align*}
    \priority{1} \geq \priority{2} \geq \cdots \geq \priority{G} \ .
\end{align*}
However, when $\effortvec$ does not follow perfect rank prioritization, we need a metric to measure the extent to which groups are appropriately prioritized. This metric would quantify similarity between the ranking induced by the benefits of $\effortvec$ and the true $\truerank$. One common approach to evaluate the similarity between rankings is the \emph{Kendall tau} coefficient:
\begin{align}
    \frac{(\text{\# concordant pairs}) - (\text{\# discordant pairs})}{\binom{G}{2}} \ .
\end{align}
The number of concordant pairs is the number of pairwise rankings that agree, across all $\binom{G}{2}$ pairwise rankings; the discordant pairs are those that disagree. This metric yields a value $[-1, 1]$ which signals amount of agreement, where $+1$ is perfect agreement 
and $-1$ is complete disagreement (i.e., reverse ordering). However, a critical weakness of Kendall tau is that it is discontinuous, abruptly jumping in value when the effort~$\effort{i}$ on an arm is perturbed,
rendering optimization difficult. 

Fortunately, we have at our disposal not just each pairwise ranking, but also the \emph{magnitude} by which each pair is in concordance or discordance. Leveraging this information, we construct a convex function $\dist(\effortvec)$ that quantifies group prioritization:
\begin{align}
\label{eq:fairness-unsimplified}
\resizebox{.9\linewidth}{!}{
    $\dist(\effortvec) = \frac{\sum\limits_{g, h \in [G]} \mathds{1}(g < h) \cdot \left( \priority{g} - \priority{h} \right) }{\binom{G}{2}}$
}
\end{align}
where $\mathds{1}( g < h )=1$ if $g$ is of higher priority than $h$ according to the true $\truerank$, and $0$ otherwise. When $g < h$, the difference in benefit to group~$g$ and~$h$ is positive if the pairs are concordant and negative if the pairs are discordant. Intuitively, for each concordant pair, we earn $(\priority{g} - \priority{h})$, and for each discordant pair, we are penalized this quantity. 
The value~$\dist(\effortvec)$ can be interpreted as the correlation between the priority outcomes of $\effortvec$ and the true priority ranking, giving us a metric for evaluating the group prioritization of an action $\effortvec$.

Our prioritization metric $\dist(\effortvec)$ induces nice properties. 
First, this metric is robust to perturbations in $\truereward{i}(\effort{i})$ values whereas Kendall tau is not.
As $\dist(\cdot)$ is continuous and differentiable, we can use it directly in the objective function. Second, with slight modifications to $\dist$ we can seamlessly adapt to other useful settings, as we discuss in Section~\ref{sec:calibrated-fairness}. 
Leveraging the fixed ordering $\truerank = \langle 1, \ldots, G  \rangle$ over groups, we can simplify Equation~\eqref{eq:fairness-unsimplified} as
\begin{align}
    \dist(\effortvec) &= \frac{\sum\limits_{g=1}^{G-1} \sum\limits_{h=g+1}^{G} \left( \priority{g} - \priority{h} \right) }{\binom{G}{2}} 
    \nonumber \\
    &\triangleq \frac{\sum\limits_{g=1}^{G-1} \sum\limits_{h=g+1}^{G} \left(\sum\limits_{i=1}^{N}\truereward{i}(\effort{i}) d_{gi}  - \sum\limits_{i=1}^{N} \truereward{i}(\effort{i}) d_{hi}\right) }{\binom{G}{2}}\nonumber\\
    &=\sum_{i=1}^N \truereward{i}(\effort{i}) \frac{\sum\limits_{g=1}^{G-1} \sum\limits_{h=g+1}^{G} \left(d_{gi}  -  d_{hi}\right) }{\binom{G}{2}} \ . \label{eq:fairness-concrete}
\end{align} 

\subsection{Objective with prioritization and reward}

We wish to take actions that maximize our expected reward (Eq.~\ref{eq:expected_reward}) while also distributing our effort across the various groups as effectively as possible (Eq.~\ref{eq:fairness-concrete}). Recognizing that targeting group prioritization requires us to sacrifice reward, we set up the objective to balance reward and prioritization with a parameter $\tunepriority \in (0, 1]$ that enables us to tune the degree to which we emphasize reward vs.\ prioritization:
\begin{align}
    \label{eq:obj-abstract}
    \obj = \tunepriority 
    \truerewardvec(\effortvec)
    + (1 - \tunepriority) \cdot \dist(\effortvec) \ .
\end{align}



\section{Approach}
We show that the objective (Eq.~\ref{eq:obj-abstract}) can be reformulated as a weighted linear combination of the reward of each individual arm, which we then solve using our RankedCUCB algorithm, producing a general-form solution for combinatorial bandits. 

Using the prioritization metric from Equation~\eqref{eq:fairness-concrete}, we can re-express the objective~\eqref{eq:obj-abstract} as
\begin{align}
\resizebox{.91\linewidth}{!}{
    $\obj = \sum\limits_{i=1}^{N} \truereward{i}(\effort{i}) \left( \tunepriority 
    + (1 - \tunepriority) 
    \frac{
    \sum\limits_{g=1}^{G-1} \sum\limits_{h=g+1}^{G} 
    ( \density{gi} - \density{hi} ) 
    } 
    { \binom{G}{2} } \right)$
}
\end{align}
Observe that the large parenthetical is comprised only of known constants, so we can precompute the parenthetical value as $\Gamma_i$ for each location~$i$ to yield
\begin{align}
\label{eq:obj-simple}
    \obj &= \sum_{i=1}^{N} \truereward{i}(\effort{i}) \Gamma_i \ ,
\end{align}
providing the useful result that our objective is decomposable across the locations $i \in [N]$ --- specifically that the objective is a linear combination of the reward for each component of the action. 


\subsection{Combinatorial UCB approach}
The expected reward functions~$\truereward{i}(\cdot)$ are unknown in advance and stochastic, so our setting falls within the space of sequential learning problems that need to be learned over multiple timesteps. Therefore, an action (an effort vector~$\effortvec$) we take in each timestep can be viewed as an arm pull under a multi-armed bandit framework. 
Specifically, we use a combinatorial bandit approach \citep{chen2016combinatorial}, where each \emph{arm} is an effort level~$\discretizationlevel{j}$ at one location~$i$, rendering $NJ$ total arms. We track the reward of each arm separately then use an oracle to select a combined \emph{action} $\effortvec$, which is a combinatorial selection of arms. 
Our goal is to design an algorithm that determines an action at each time~$t$ to minimize regret, defined as the difference between the total expected reward of an optimal action and the reward achieved following actions recommended by the algorithm. 

To balance exploration (learning the reward of actions with greater uncertainty) with exploitation (taking actions known to be high-reward), we employ an upper confidence bound (UCB) approach that considers for each action its estimated reward along with our uncertainty over that estimate.
Let $\approxreward{t}(i, j) = \observedreward_t(i, \discretizationlevel{j}) / n_t(i, \discretizationlevel{j})$ 
be the average observed reward of location~$i$ at effort~$\discretizationlevel{j}$ given cumulative empirical reward $\observedreward_t(i, \discretizationlevel{j})$ over $n_t(i, \discretizationlevel{j})$ arm pulls by timestep~$t$. 
The \emph{confidence radius}~$r_t$ of an arm $(i, j)$ is then a function of the number of times we have pulled that arm:
\begin{align}
    r_t(i,j) = \sqrt{\frac{3 \Gamma_i^2 \log T }{2 n_t(i, \discretizationlevel{j})}} \ .
\end{align}

We distinguish between UCB and a term we call $\selfucb$. The $\selfucb$ of an arm $(i, j)$ representing location~$i$ with effort~$\discretizationlevel{j}$ at time~$t$ is the UCB of an arm based only on its own observations, given by
\begin{align}
    \selfucb_t(i, j) = \Gamma_i \hat{\mu}_t(i, j) + r_t(i, j) \ .
\end{align}
This definition of $\selfucb$ corresponds with the standard interpretation of confidence bounds from the standard UCB1 algorithm \citep{auer2002finite}. The UCB of an arm is then computed by taking the minimum of the bounds of all $\selfucb$s as applied to the arm. These bounds are determined by adding the distance between arm $(i, j)$ and all other arms $(u, v)$ to the $\selfucb$:
\begin{align}
\label{eq:ucb}
\ucb_t(i, j) &= \min\limits_{\substack{k \in [J]}} \{ \selfucb_t(i, k) + L \cdot \textit{dist} \} \\ 
\textit{dist} &= \max\{0,  \discretizationlevel{j} - \discretizationlevel{k} \} 
\end{align}
which leverages the assumptions described in Section~\ref{sec:model} that the expected rewards are $L$-Lipschitz continuous and monotonicially nondecreasing.

Given these UCB estimates for each arm $(i,j)$, we must now select a combinatorial action $\effortvec$ to take at each timestep. 
As the prioritization metric can be expressed as a linear combination of the reward~\eqref{eq:obj-simple}, we can directly optimize the overall objective using an integer linear program $\mathcal{LP}$ (in Appendix~\ref{sec:lp}), which selects an optimal action that respects the budget constraint.




\subsection{Trading off prioritization and learning}

A key challenge with our prioritization metric is that it is defined with respect to the true expected reward functions $\truereward{i}$ which are initially unknown. Instead, we estimate per-round prioritization based on our current \emph{belief} of the true reward, $\approxreward{i}$, related to \emph{subjective fairness} from the algorithmic fairness literature \citep{dimitrakakis2017subjective}. However, this belief is nonsensical in early rounds when our reward estimates are extremely coarse, so we discount the weight of prioritization in our objective until our learning improves. 

Inspired by decayed epsilon-greedy, we incorporate an $\epsilon$ coefficient to tune how much we want to prioritize rank order at each step (versus learning to maximize reward):
\begin{align}
    \label{eq:obj-linear-epsilon}
    \obj = \tunepriority 
    \truerewardvec(\effortvec) + (1 - \tunepriority) (1 - \epsilon) \cdot \dist(\effortvec) \ .
\end{align}
For example, epsilon-greedy methods often use exploration probability $\epsilon_t \sim t^{-1/3}$, which gradually attenuates at a decreasing rate with each increased timestep. 
Our definition gives nice properties that $\epsilon_t = 1$ at $t = 1$, so we do not care about ranked prioritization at all (we have no estimate of the reward values so we cannot reasonably estimate ranking) but at $t = \infty$ then $\epsilon_t = 0$ and we fully care about ranking w.r.t.\ $\tunepriority$ (since we have full knowledge of the reward thus can precisely estimate rank order). 

All together, we call the approach described here RankedCUCB and provide pseudocode in Algorithm~\ref{alg-ranked-cucb}.


\begin{algorithm}[tb]
\caption{RankedCUCB}
\label{alg-ranked-cucb}
\raggedright
\textbf{Input}: time horizon~$T$, budget~$B$, discretization levels $\discretization = \{ \discretizationlevel{1}, \ldots, \discretizationlevel{J} \}$, arms~$i \in [N]$ with unknown reward~$\truereward{i}$ \\
\textbf{Parameters}: tuning parameter $\tunepriority$ \\ 
\begin{algorithmic}[1] 
\STATE Precompute $\Gamma_i$ for each arm $i \in [N]$
\STATE $n(i, \psi_j) = 0, \texttt{reward}(i, \psi_j) = 0 \quad \forall i \in [N], j \in [J]$
\FOR{timestep $t = 1, 2, \ldots, T$}
\STATE Let $\epsilon_t = t^{-1/3}$
\STATE Compute $\text{UCB}_t$ for all arms using Eq.~\eqref{eq:ucb}
\STATE Solve $\mathcal{LP}(\ucb_t, \{\Gamma_i\}, B)$ to select super arm~$\effortvec$
\STATE \textit{// Execute action}
\STATE Act on $\effortvec^\prime$ to observe rewards $X_1^{(t)}, X_2^{(t)}, \ldots, X_N^{(t)}$ \\
\FOR {arm $i = 1, \ldots, N$}
\STATE $\texttt{reward}(i, \effort{i}^\prime) = \texttt{reward}(i, \effort{i}^\prime) + X_i^{(t)}$  \\
\STATE $n(i, \effort{i}) = n(i, \effort{i}) + 1$
\ENDFOR
\ENDFOR
\end{algorithmic}
\end{algorithm}

\section{Regret analysis}
We now prove that our iterative algorithm RankedCUCB (Alg.~\ref{alg-ranked-cucb}) guarantees no regret with respect to the optimal solution for the objective~\eqref{eq:obj-simple} that jointly considers reward and prioritization. More formally, we show that $\text{Regret}_T\rightarrow 0$ as $T\rightarrow \infty$ where
\begin{align}
   \text{Regret}_T := \truerewardvec(\opteffortvec) - \frac{1}{T}\sum_{t=1}^T \truerewardvec(\effortvec^{(t)}) \ .  \label{eq:regret}
\end{align}
Here, $\opteffortvec$ is an optimal action and expected reward $\truerewardvec(\effortvec):= \sum_{i=1}^N \truereward{i}(\effort{i}) \Gamma_i$ for an effort vector $\effortvec = \{\effort{1}, \ldots, \effort{N} \}$. Note that if $\Gamma_i<0$, any solution to the maximization problem~\eqref{eq:obj-simple} would allocate $\beta_i=0$, and so would RankedCUCB. Hence, for the analysis we assume that we consider only those locations whose $\Gamma_i>0$.

\subsection{Convergence of estimates}
To prove the no-regret guarantee of RankedCUCB, we first establish Lemma~\ref{lemma:tau-turns} which states that, with high probability, the UCB estimate $\hat{\mu}_t({i,j}) \Gamma_i + r_t(i, j)$ converges within a confidence bound of $r_t(i,j)=\sqrt{ (3 \Gamma_i^2 \log T) / (2 n_t)}.$

\begin{restatable}[]{lemma}{tauTurns}
\label{lemma:tau-turns}
Using RankedCUCB, after $t$ timesteps, 
each $\hat{\truerewardvec}_t(i,j) \Gamma_i$ estimate converges to a value within a radius $r_t (i, j) = \sqrt{3\Gamma_i^2 \ln{t} / 2 n_t(i,j)}$ of the corresponding true $\truerewardvec_t(i,j)\Gamma_i$ values with probability $1 - \frac{2NJ}{t^2}$ for all $i, j$. 
\end{restatable}
\begin{proof}[Proof sketch]
Using the Chernoff--Hoeffding bound, we show that, at any timestep $t$, the probability that the difference between $\mu_t(i,j)\Gamma_i$ and $\hat{\mu}_t(i,j)\Gamma_i$ is greater than $r_t(i,j)$, is at most $2/t^3$. We then use union bound to show that $\hat{\mu}_t(i,j)\Gamma_i$ converges to a value within radius $r_t(i,j)$ of $\mu_t(i,j)\Gamma_i$ with probability $1-\frac{2NJ}{t^2}$. The complete proof is given in Appendix~\ref{appendix:proof-convergence}.
\end{proof}

\begin{figure*}[ht]
\centering
  \input{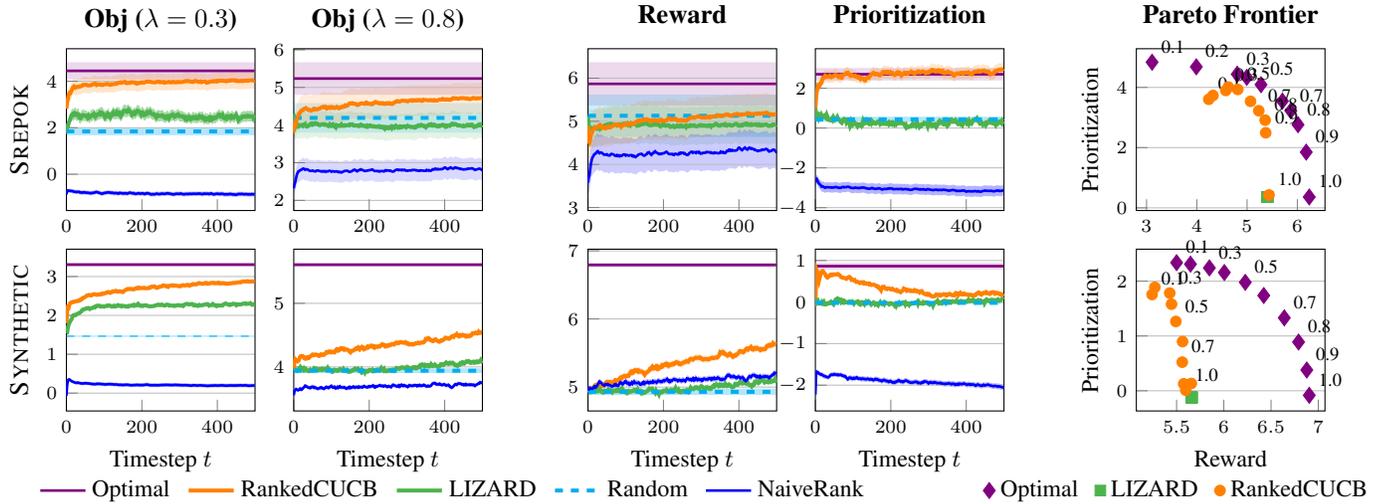}
  \caption{The performance of each approach. \textbf{LEFT} evaluates the objective with tuning parameter $\tunepriority = 0.3$ and $\tunepriority = 0.8$. Our approach, RankedCUCB, performs significantly better than baselines.  \textbf{CENTER} evaluates reward and prioritization (at $\tunepriority = 0.8$), the two components of the combined objective. The reward-maximizing LIZARD algorithm rapidly maximizes reward but performs worse than random in terms of rank order. 
  \textbf{RIGHT} visualizes the Pareto frontier trading off the two components of our objective. Labels represent different values of $\tunepriority \in \{0.1, 0.2, \ldots, 1.0\}$. Each point plots the reward and ranked prioritization as the average of the final 10 timesteps. All results are averaged over 30 random seeds and smoothed with an exponential moving average.}
    \label{fig:performance}
\end{figure*}

\subsection{Achieving no regret}

\begin{restatable}[]{theorem}{regretTheorem}
\label{theorem:regret}
The cumulative regret of RankedCUCB is $O \left( \frac{J \ln{T}}{N} + NJ \right)$ with $N$~arms, $J$~discrete effort values, and time horizon~$T$.
\end{restatable}
\begin{proof}[Proof sketch]
We first show the expected regret~(Eq.~\ref{eq:regret}) of RankedCUCB can be redefined as:
\[\frac{1}{T}\sum_{i=1}^N \sum_{j=1}^J \mathbb{E}[L_{i,j,T}]\ \zeta_{i,j}\]
where $L_{i,j,T}$ specifies the number of times the effort $\discretizationlevel{j}$ selected for location~$i$ results in a suboptimal solution, and $\zeta_{i,j}$ denotes the minimum loss incurred due to the sub-optimal selection. Then, by contradiction, we show that when all $\hat{\truerewardvec}_t(i,j)\Gamma_i$ estimates are within their confidence radius, RankedCUCB selects an optimal effort $\discretizationlevel{j}$ (and not a sub-optimal one) at time~$t$. Using this fact and Lemma~\ref{lemma:tau-turns}, we show that the $\mathbb{E}[L_{i,j,T}]$ can be upper bounded by a finite term, and hence the expected regret is $O (\frac{J \ln{T}}{N} + NJ )$.  
The complete proof is given in Appendix~\ref{appendix:proof-regret}.
\end{proof}

Our analysis improves upon the regret analysis for standard (non-weighted) combinatorial bandits \citep{chen2016combinatorial} by not  having to rely on the ``bounded smoothness'' assumption of reward functions. 
The technical tools we employ for this regret analysis can be of independent interest to the general multi-armed bandits community. 

\section{Experiments}
\label{sec:experiments}


We evaluate the performance of our RankedCUCB approach in real-world conservation data from Srepok Wildlife Sanctuary to demonstrate significant improvement over approaches that ignore prioritization.

Of the key species in Srepok on which we have species density data, we prioritize elephants, followed by banteng, muntjac, and wild pig, corresponding to their IUCN Red List categories of critically endangered, endangered, least concern but decreasing, and least concern respectively. 
Thus our setting has $G=4$ groups distributed across $N=25$ locations within the park. We further evaluate our algorithm on a synthetic dataset with $G=5$ groups randomly distributed across the park.

\paragraph{Baselines}
We compare the performance of our \textit{RankedCUCB} approach with naive priority-aware, priority-blind, random, and optimal benchmarks. 
\textit{NaiveRank} takes the straightforward approach of directly solving for the objective that weighs each target by its individual ranked prioritarian metric, which accounts for prioritization induced by each target independently but ignores the coupled effect across targets. 
\textit{LIZARD} \citep{xu2021dual} maximizes the combinatorial reward objective; this algorithm enhances standard UCB approaches by accounting for smoothness and monotonicity in the reward function, but ignores prioritization. 
\textit{Optimal} solves for the action $\opteffortvec = \arg \max_{\effortvec} \obj$ that maximizes the objective, using the ground truth values of $\truereward{i}$ to directly compute the best action. \textit{Random} selects an arbitrary subset of actions that satisfies the budget constraint. 




See Figure~\ref{fig:performance} for a comparison of the approaches. RankedCUCB performs consistently the best on our overall objective, and the breakdown of the reward and prioritization components reveals that this gain is a result of the tradeoff between the two components: although LIZARD is able to learn high-reward actions, these actions lead to prioritization outcomes that are worse than random. LIZARD even achieves reward that exceeds that of the optimal action (which also considers rank priority), but that comes at the cost of poor prioritization. Notably, NaiveRank performs worse than random, even measured on fairness: focusing on the individual targets with the best prioritization neglects group-wide patterns throughout the park.



The reward--prioritization tradeoff becomes more apparent when we analyze the Pareto frontier in Figure~\ref{fig:performance}(right), which plots the reward and prioritization of each approach as we change the value of $\tunepriority$. The price of prioritization is clear: the more we enforce prioritization (smaller values of $\tunepriority$) the lower our possible reward. However, this tradeoff is not one-for-one; the steepness of the Pareto curve allows us to strategically select the value of $\tunepriority$. For example, we achieve a nearly two-fold increase in prioritization by going from $\tunepriority=0.9$ to $0.7$ while sacrificing less than 10\% of the reward, indicating a worthwhile tradeoff. 


\section{Generalizability to other settings}
\label{sec:calibrated-fairness}

Beyond wildlife conservation, our problem formulation applies to domains with the following characteristics: (1)~repeated resource allocation, (2)~multiple groups with some priority ordering, (3)~an \textit{a priori} unknown reward function that must be learned over time, and (4)~actions that impact some subset of the groups to differing degrees. Our approach also adapts to the following related objectives.

\paragraph{Weighted rank} 
\citet{gatmiry2021security} suggest some specific metrics for how rangers should prioritize different species, noting ranger enforcement of poaching urial (a wild sheep) should be 2.5 stricter than red deer. 
Our approach could easily adapt to  these settings where domain experts have more precise specification of desired outcomes. In this case, we introduce a parameter $\alpha_g$ for each group~$g$ that specifies the relative importance of group~$g$, leading to the following prioritization metric:
\begin{align}
\label{eq:calibrated-fairness}
    \dist_\text{CF}(\effortvec) = \sum\limits_{i=1}^{N} \truereward{i}(\effort{i}) \frac{
    \sum\limits_{g=1}^{G-1} \sum\limits_{h=g+1}^{G} 
    \alpha_g \density{gi} - \alpha_h \density{hi}
    } 
    { \binom{G}{2} }
\end{align}
where we set $\alpha_g > \alpha_h$ when $g < h$ to raise our threshold of how strongly group~$g$ should be favored. From our example, we would set $\alpha_{\text{urial}} = 2.5 \alpha_{\text{deer}}$.
This metric aligns with the definition of \emph{calibrated fairness} \citep{liu2017calibrated}.




\paragraph{Weighted reward}
We can also accommodate the setting where the reward is a weighted combination, i.e., if reward is $\sum_{i=1}^{N} c_i \truereward{i}(\effort{i})$ for some coefficients~$c_i \in \mathbb{R}$, as the coefficients~$c_i$ would be absorbed into the $\Gamma_i$ term of the objective.

\section{Related work}


\paragraph{Multi-armed bandits} MABs \citep{lattimore2020bandit} have been applied to resource allocation for healthcare \citep{bastani2020online}, education \citep{segal2018combining}, 
and dynamic pricing \citep{misra2019dynamic}. 
These papers solve various versions of the stochastic MAB problem \citep{auer2002finite}.

Several prior works consider resource allocation settings where each arm pull is costly, limiting the total number of pulls by a budget. \citet{tran2010epsilon} use an $\epsilon$--based approach to achieve regret linear in the budget $B$; \citet{tran2012knapsack} provide a knapsack-based policy using UCB to improve to $O(\ln B)$ regret. \citet{chen2016combinatorial} extend UCB1 to a combinatorial setting, matching the regret bound of $O(\log T)$. 
\citet{slivkins2013dynamic} address ad allocation where each stochastic arm is limited by a budget constraint on the number of pulls. Our objective~\eqref{eq:obj-simple} is similar to theirs but they do not handle combinatorial actions, which makes our regret analysis even harder. 

\paragraph{Prioritization notions in MABs}
Addressing prioritization requires reward trade-offs, which is related to objectives from algorithmic fairness \citep{dwork2012fairness,kleinberg2018algorithmic,corbett2017algorithmic}. However, most MAB literature favor traditional reward-maximizing approaches \citep{,lattimore2015linear,agrawal2014bandits,verma2019censored}. 
The closest metric to our prioritization objective is \emph{meritocratic fairness} \citep{kearns2017meritocratic}. For MAB, meritocratic fairness was introduced by \citet{joseph2016fairness}, to ensure that lower-reward arms 
are never favored above a high-reward one. 
\citet{liu2017calibrated} apply calibrated fairness by enforcing a smoothness constraint on our selection of similar arms, using Thompson sampling to achieve an $\tilde{O}(T^{2/3})$ fairness regret bound. 
For contextual bandits, \citet{chohlaslearning} take a consequentialist (outcome-oriented) approach and evaluate possible outcomes on a Pareto frontier, which we also investigate. 
\citet{wang2021fairness} ensure that arms are proportionally represented according to their merit, a metric they call \emph{fairness of exposure}. 
Others have considered fairness in multi-agent bandits \citep{hossain2021fair}, restless bandits \citep{herlihy2021planning}, sleeping bandits \citep{li2019combinatorial}, contextual bandits with biased feedback \citep{schumann2022group}, and infinite bandits~\citep{joseph2018meritocratic}.
However, none of these papers consider correlated groups, which is our focus.

\section{Conclusion}
We address the challenge of allocating limited resources to prioritize endangered species within protected areas where true snare distributions are unknown, closing a key gap identified by our conservation partners at Srepok Wildlife Sanctuary. 
Our novel problem formulation introduces the metric of \emph{ranked prioritization} to measure impact across disparate groups. 
Our RankedCUCB algorithm offers a principled way of balancing the tradeoff between reward and ranked prioritization in an online learning setting. Notably, our theoretical guarantee bounding the regret of RankedCUCB applies to a broad class of combinatorial bandit problems with a weighted linear objective. 

\section*{Acknowledgments}

This work was supported in part by NSF grant IIS-2046640 (CAREER) and the Army Research Office (MURI W911NF-18-1-0208). 
Biswas supported by the Center for Research on Computation and Society (CRCS).
Thank you to Andrew Plumptre for a helpful discussion on species prioritization in wildlife conservation; Jessie Finocchiaro for comments on an earlier draft; and all the rangers on the front lines of biodiversity efforts in Srepok and around the world.


\bibliographystyle{named}
\bibliography{short,ref}

\clearpage
\appendix
\section{Linear program}
\label{sec:lp}
The following linear program $\mathcal{LP}$ takes the current $\ucb_t(i,j)$ estimates to select a combinatorial action that maximizes our optimistic reward. 

Each $z_{i,j}$ is an indicator variable specifying whether we choose effort level $\discretizationlevel{j}$ for location~$i$. 

\begin{align*}
    \max_{z} &\sum_{i=1}^{N} \sum_{j=1}^{J}  z_{i,j} \ucb_{t}(i, j) 
    \tag{$\mathcal{LP}$} \\
    \text{s.t. } & z_{i,j} \in \{0,1\} \qquad \forall i \in [N], j \in [J] \\
    & \sum_{j=1}^{J} z_{i,j} = 1 \qquad \forall i \in [N] \\
    & \sum_{i=1}^{N} \sum_{j=1}^{J} z_{i,j} \psi_j \leq B \ .
\end{align*}

\section{Full Proofs}
\label{appendix:proofs}

\subsection{Notation}

For ease of representation, we use  $\rho_{i,j}^{(t)} := \mu_t(i,j)\Gamma_i$ to denote the expected reward of visiting location~$i$ with effort $\discretizationlevel{j}$ and $\hat{\rho}_{i,j}^{(t)} := \hat{\mu}_t({i,j}) \Gamma_i$ as the average empirical reward (the index~$j$ denotes one that corresponds to the discretization level of $\effort{i}$, i.e., $\effort{i} = \discretizationlevel{j}$). 
We use $\hat{\rho}$ to denote the empirical average reward and $\overline{\rho}$ to denote the upper confidence bound. Note that $\overline{\rho} = \hat{\rho} + r_t(i, j)$. Similarly for $\hat{\mu}$ and $\overline{\mu}$. 

\subsection{Convergence of estimates}
\label{appendix:proof-convergence}
\tauTurns*

\begin{proof}
In other words, we wish to show that each $\hat{\rho}^{(t)}_{i,j}$ estimate converges to a value within radius $r_t(i,j)$ of the corresponding true $\rho_{i,j}^{(t)}$ values.

We provide a lower bound on the probability that estimated $\approxrewardij{i}{j}$ values are within a bounded radius for all $i \in [N]$ and all $j \in [J]$. At time~$t$,
\begin{align}
\mathbb{P} \left( \left| \rho_{i,j} - \hat{\rho}_{i,j}^{(t)} \right| \leq r_t(i,j), \ \ \forall i \in [N]\  \forall j \in [J] \right) \nonumber \\
   = 1 - \sum_{i=1}^{N}\sum_{j=1}^{J} \mathbb{P} \left( \left| \rho_{i,j} - \hat{\rho}_{i,j}^{(t)} \right| > r_t(i,j) \right)  \label{eq:union-bound}
\end{align}
as the events are all independent. Assume that the number of samples used for computing the estimate $\hat{\rho}_{i,j}$ is $n_t(i,j)$. Using the fact that ${\rho}_{i,j} \in [0, \Gamma_i]$ and the Chernoff--Hoeffding bound, we find that
\begin{align}
  \mathbb{P} \left( \left| \rho_{i,j} - \hat{\rho}_{i,j}^{(t)} \right| > r_t(i,j) \right)
   &\leq 2 \exp \left( - \frac{2n_t(i,j)}{\Gamma_i^2} r_t(i,j)^2 \right) \nonumber \\
   &= 2 \exp \left( -{3 \ln{t}}\right) \tag{obtained by substituting value of $r_t(i,j)$} \\
   &= \frac{2}{t^3} \ . \label{eq:chernoff-hoeffding}
\end{align}
By substituting Eq.~\eqref{eq:chernoff-hoeffding} in Eq.~\eqref{eq:union-bound}, we obtain
\begin{align}
   &1-\sum_{s=1}^t\sum_{i=1}^{N}\sum_{j=1}^{J} \mathbb{P} \left( \left| \rho_{i,j} - \hat{\rho}_{i,j}^{(s)} \right| > \sqrt{3\Gamma_i^2 \ln{t} / 2s} \right) \nonumber \\
  &\geq 1 - \sum_{s=1}^t\sum_{i=1}^{N}\sum_{j=1}^{J} \frac{2}{t^3} \nonumber \\
  &=1 - \frac{2NJ}{t^2} \ . \label{eq:lower-bound}
\end{align}

This completes the proof.
\end{proof}

\subsection{Regret bound}
\label{appendix:proof-regret} 
\regretTheorem*

\begin{proof}
For a finite time horizon~$T$, the average cumulative regret of an algorithm that takes action $\effortvec^{(t)}$ at time~$t$ is given by Eq.~\eqref{eq:regret}. We use $\mathcal{B}$ to denote the set of all sub-optimal actions:
\begin{align*}
\mathcal{B} := \{\effortvec \mid \truerewardvec(\effortvec) < \truerewardvec(\opteffortvec)\} \ .
\end{align*}
We now define regret as the expected loss incurred from choosing $\effortvec$ from the set $\mathcal{B}$. Thus, Eq.~\eqref{eq:regret} can be written as
\begin{align}   
   \sum_{i=1}^N \truereward{i}(\opteffort{i}) \Gamma_i - \frac{1}{T}\sum_{t=1}^T\mathbb{E}\left[\sum_{i=1}^N \mu_i ( \effort{i}^{(t)} ) \right] \Gamma_i 
   \leq \frac{1}{T}\sum_{i=1}^N \sum_{j=1}^J \mathbb{E}[L_{i,j,T}]\ \zeta_{i,j} \label{eq:regret2}
\end{align}
where $L_{i,j,T}$ denotes the number of times the effort for arm~$i$ is set as $\effort{i} = \discretizationlevel{j}$ and the corresponding $\effortvec \in \mathcal{B}$. That is, $L_{i,j,T}$ specifies the number of times the pair $(i,j)$ is chosen in a suboptimal way. Note that $L_{i,j,T} \leq n_T(i,j)$. Let $\zeta_{i,j}$ denote the minimum loss incurred due to a sub-optimal selection (of effort~$\discretizationlevel{j}$) on arm~$i$. In other words,
\begin{align*}
\zeta_{i,j} = \truerewardvec(\opteffortvec) - \max_{i,j}\{\truereward{i}(\discretizationlevel{j}) \mid \effortvec \in \mathcal{B} \text{ and } \effort{i} = \discretizationlevel{j} \} \ .    
\end{align*}

Let ${\zeta}_{\min} = \min_{i,j}\zeta_{i,j}$ and $\tau_{t} = \frac{6\Gamma^2_i \ln{t}}{N^2\zeta_{\min}^2}$. Let us assume that, at time step~$t$, all arms have been visited at least $\tau_t$ times, so $n_t(i, j) \geq \tau_t$ for all $i \in [N]$ and $j \in [J]$. Using contradiction, we show that the algorithm will not choose any sub-optimal vector $\effortvec \in \mathcal{B}$ at time~$t$ when the $\hat{\rho}_{i,j}^{(t)}$ estimates converge to a value within a radius $r_t(i,j)=\sqrt{3\Gamma_i^2\ln{t}/2n_t(i,j)}$ of $\rho_{i,j}$:
\begin{align}
  R_{\overline{\mu}}(\effortvec)&= \sum_{i=1}^N \hat{\rho}_{i,j}^{(t)}+ \sqrt{\frac{3\Gamma_i^2 \ln{t}}{2 n_t(i,j)}}\nonumber\\
  &\leq \sum_{i=1}^N  {\rho}_{i,j}^{(t)} + 2\sqrt{\frac{3\Gamma_i^2 \ln{t}}{2 n_t(i,j)}}\nonumber\\
  &\tag{since estimates are within radius} \\
  &\leq \sum_{i=1}^N {\rho}_{i,j}^{(t)} +\sqrt{\frac{6\Gamma_i^2 \ln{t}}{\tau_t}} \tag{since $n_{t}(i,j) \geq \tau_t$}\nonumber \\
  &\leq \zeta_{\min} + \sum_{i=1}^N {\rho}_{i,j}^{(t)} \ . \qquad\qquad\qquad \text{(substituting $\tau_t$)} \label{eq:upper}
\end{align}
However, the assumption that $\effortvec$ was chosen by the algorithm at time~$t$ implies that
\begin{align}
   \overline{\truerewardvec}(\effortvec)&>  \overline{\truerewardvec}(\opteffortvec)\nonumber\\
    &\geq \hat{\truerewardvec}(\opteffortvec) + \sum_{i=1}^N \sqrt{\frac{3\Gamma_i^2 \ln{t} }{2 n_t(i,j^\star)}}
    \tag{$j^\star$ is the effort level corresponding to $\opteffortvec_i$} \\
   &\geq \truerewardvec(\opteffortvec) \ . \qquad \text{(since estimates are within radius)} \label{eq:lower}
\end{align}
By combining Eqs.~\eqref{eq:upper} and~\eqref{eq:lower} we obtain,
\begin{align}
    \zeta_{\min} + \sum_{i=1}^N {\rho}_{i,j}^{(t)} &\geq \truerewardvec(\opteffortvec)\nonumber\\
    \zeta_{\min} &\geq \truerewardvec(\opteffortvec)- \truerewardvec(\effortvec^{(t)}) \ . \label{eq:contradiction}
\end{align}
Inequality~\eqref{eq:contradiction} contradicts the definition of $\zeta_{\min}$. Thus, the algorithm selects an optimal effort (and not the sub-optimal ones) at time~$t$ when all the  $\overline{\rho}^{(t)}_{i,j}$ estimates are within the radius $r_t(i,j) = \sqrt{3\Gamma_i^2 \ln{t}/2n_t(i,j)}$ of the true $\rho_{i,j}$. Using this fact along with Lemma~\ref{lemma:tau-turns}, we obtain an upper bound on $\mathbb{E}[L_{i,j,T}]$:
\begin{align}
   \mathbb{E}[L_{i,j,T}] &\leq (\tau_T + 1) NJ + \sum_{t=1}^T \frac{2NJ}{t^2}\nonumber \\
   &\leq (\tau_T + 1)NJ + \frac{\pi^2}{3}NJ \ .
\end{align}

We use ${\zeta}_{\max} = \max_{i,j}\zeta_{i,j}$ to provide a bound on Eq.~\eqref{eq:regret2}, which is $O(N)$. This produces the following regret for RankedCUCB by time~$T$:
\begin{align}
   \text{Regret}_T &\leq \frac{1}{T} \left( \tau_T + 1 + \frac{\pi^2}{3} \right) NJ {\zeta}_{\max} \nonumber\\
   &= \left( \frac{6\Gamma_i^2\ln{T}}{N^2\zeta_{\min}^2} + 1 + \frac{\pi^2}{3} \right) \frac{NJ {\zeta}_{\max}}{T} \ .
\end{align}
This result shows that the worst case regret increases linearly with the increase in the number of discrete effort levels. However, the regret grows only sub-linearly with time~$T$. This proves that RankedCUCB asymptotically converges as a no-regret algorithm, that is, as $T\rightarrow\infty$, the regret of RankedCUCB tends to $0$.
\end{proof}

\end{document}